\documentclass[11pt]{article}
\pdfoutput=1

\usepackage{acl}

\usepackage{times}
\usepackage{latexsym}

\usepackage[T1]{fontenc}

\usepackage[utf8]{inputenc}

\usepackage{microtype}

\usepackage{inconsolata}
\usepackage{amsmath}   
\usepackage{amsthm}    

\usepackage{cleveref}  
\theoremstyle{plain}
\newtheorem{theorem}{Theorem}

\usepackage{graphicx}
\usepackage{subcaption}
\usepackage{float}
\usepackage{amsmath}
\usepackage{mathtools}   
\usepackage{multirow}
\usepackage{amsfonts}
\usepackage{adjustbox}
\usepackage{booktabs}
\usepackage{pifont}
\usepackage{enumitem}
\usepackage{colortbl}
\definecolor{lightred}{RGB}{255,230,230}
\definecolor{lightgreen}{RGB}{230,255,230}

\title{EffiReason-Bench: A Unified Benchmark for Evaluating and Advancing Efficient Reasoning in Large Language Models}

\author{
  \textbf{Junquan Huang\textsuperscript{1,*}},
  \textbf{Haotian Wu\textsuperscript{2,*}},
  \textbf{Yubo Gao\textsuperscript{1,*}},
  \textbf{Yibo Yan\textsuperscript{1,3}},
\\
  \textbf{Junyan Zhang\textsuperscript{1}},
  \textbf{Yonghua Hei\textsuperscript{1}},
  \textbf{Song Dai\textsuperscript{1,3}},
  \textbf{Jie Zhang\textsuperscript{2}},
  \textbf{Puay Siew Tan \textsuperscript{4}},
   \textbf{Xuming Hu\textsuperscript{1,3,$\dagger$}}
\\
  \textsuperscript{1}Hong Kong University of Science and Technology (Guangzhou),\\
  \textsuperscript{2}Nanyang Technological University,\\
  \textsuperscript{3}The Hong Kong University of Science and Technology,\\
  \textsuperscript{4}Singapore Institute of Manufacturing Technology, A*STAR,\\
  \href{mailto:xuminghu@hkust-gz.edu.cn}{xuminghu@hkust-gz.edu.cn}
}

\begin{document}
\maketitle

\begingroup
\renewcommand{\thefootnote}{\fnsymbol{footnote}} 
\footnotetext[1]{Equal contribution.}
\footnotetext[2]{Corresponding author.}
\endgroup
\setcounter{footnote}{0} 

\begin{abstract}
Large language models (LLMs) with Chain-of-Thought (CoT) prompting achieve strong reasoning but frequently generate unnecessarily verbose explanations, increasing costs and reducing accuracy. Despite the proliferation of efficiency-oriented approaches, fragmented evaluation practices hinder systematic comparison and obscure which strategies are effective under varying conditions. We present \textbf{EffiReason-Bench}, the first unified benchmark enabling rigorous cross-paradigm evaluation of efficient reasoning methods organized into three categories: \emph{Reasoning Blueprints}, \emph{Dynamic Execution}, and \emph{Post-hoc Refinement}. To enable comprehensive step-by-step reasoning evaluation, we construct verified CoT annotations for CommonsenseQA and LogiQA through a rigorous pipeline enforcing standardized reasoning structures, comprehensive option-wise analysis, and human verification. We evaluate 7 methods across 6 LLMs (1B-70B) on 4 reasoning datasets, covering mathematical, commonsense, and logical reasoning, and propose the \textbf{E$^3$-Score}, a principled metric inspired by economic trade-off modeling that provides smooth, stable evaluation without the discontinuities and over-reliance on heuristic dependencies plaguing prior measures. Experimental results demonstrate that no single method universally dominates: optimal strategies depend sensitively on backbone scale, task complexity, and architecture. 
\end{abstract}

\section{Introduction}
Large Language Models (LLMs) have shown strong capabilities in complex reasoning~\cite{chen2025towards,yan2024survey}, largely driven by the Chain-of-Thought (CoT) paradigm~\cite{wei2022chain}, where LLMs generate explicit step-by-step rationales to solve problems. This approach improves performance on logic-intensive tasks by providing structured intermediate steps~\cite{wang2023self,yao2023tree,besta2024graph}. However, it often leads to the ``overthinking phenomenon"~\cite{chen2024not,team2025kimi,feng2025efficient}, where unnecessarily long and redundant reasoning chains are produced even for simple tasks. Such verbosity not only increases computational cost but can also accumulate errors and reduce accuracy~\cite{sui2025stop}.

To mitigate these challenges, a growing body of research on \emph{efficient reasoning}~\cite{sui2025stop,wang2025harnessing,yue2025don} has proposed diverse strategies. These methods intervene at different stages of the reasoning process. Some approaches act before reasoning, serving as \emph{reasoning blueprints} that guide the LLM through reinforcement learning with length-aware rewards~\cite{shen2025dast,luo2025o1} or by using concise prompts~\cite{xu2025chain}. Others modify the generation process itself, applying \emph{dynamic execution} techniques such as reasoning in continuous space~\cite{zhang2025soft} or restructuring decoding~\cite{ning2024skeleton}. A third category, \emph{post-hoc refinement}, operates after generation, pruning, or compressing verbose outputs~\cite{xia2025tokenskip}. Though promising, the lack of a unified evaluation framework makes systematic comparison difficult. Key questions remain open: \textit{Which strategies provide the best accuracy–efficiency trade-off? How do they scale with LLM backbone size? Do their benefits transfer across reasoning domains?}

Existing benchmarks only partially address these questions. Some focus narrowly on specific method families, such as Sys2Bench~\cite{parashar2025inference} for inference-time search strategies, or on compression methods like quantization~\cite{liu2025quantization} and pruning~\cite{zhang2025reasoning}. Others analyze isolated phenomena, such as the effect of reasoning step length~\cite{jin2024impact} or the tendency to overthink in agentic tasks~\cite{cuadron2025danger}. Despite their contributions, these efforts remain fragmented. Moreover, current evaluation metrics for efficiency-effectiveness trade-off, such as Accuracy per Computation Unit (ACU)~\cite{ma2025cot} and Accuracy–Efficiency Score (AES)~\cite{luo2025o1} are limited: the former misrepresents improvements near strong baselines and the latter introduces discontinuities, and depends on excessive heuristic hyperparameters. As a result, both fair comparison and principled assessment of efficiency-effectiveness balance for efficient reasoning remain open challenges.

To fill these gaps, we introduce \textbf{EffiReason-Bench}, the first comprehensive benchmark for evaluating efficient reasoning methods across paradigms. EffiReason-Bench systematically compares seven representative methods under a consistent setup, spanning six open-source LLM backbones of varying scales and four datasets covering mathematics (\emph{GSM8K}, \emph{MATH500}), commonsense (\emph{CommonsenseQA}), and logic (\emph{LogiQA}) reasoning. Considering that the datasets CommonsenseQA and LogiQA do not provide the solution processes, to enable transparent evaluation of step-by-step reasoning processes beyond answer accuracy alone, we construct high-quality CoT annotations for CommonsenseQA and LogiQA through a rigorous three-stage pipeline: standardized reasoning structures, explicit comparative analysis across all options ensuring strict alignment with ground-truth answers, and dual human verification to eliminate logical inconsistencies.

To quantify the trade-off between efficiency and effectiveness, we further propose the \textbf{E$^3$-Score} (Efficiency–Effectiveness Equilibrium Score), a smooth and stable metric inspired by economic trade-off modeling. E$^3$-Score avoids discontinuities and heuristic tuning, emphasizes accuracy gains near strong baselines, and penalizes simultaneous degradation of efficiency and accuracy. Together, EffiReason-Bench and E$^3$-Score provide the first rigorous foundation for cross-paradigm evaluation of efficient reasoning.

\begin{figure}[htp]
    \centering
    \includegraphics[width=1.0\linewidth]{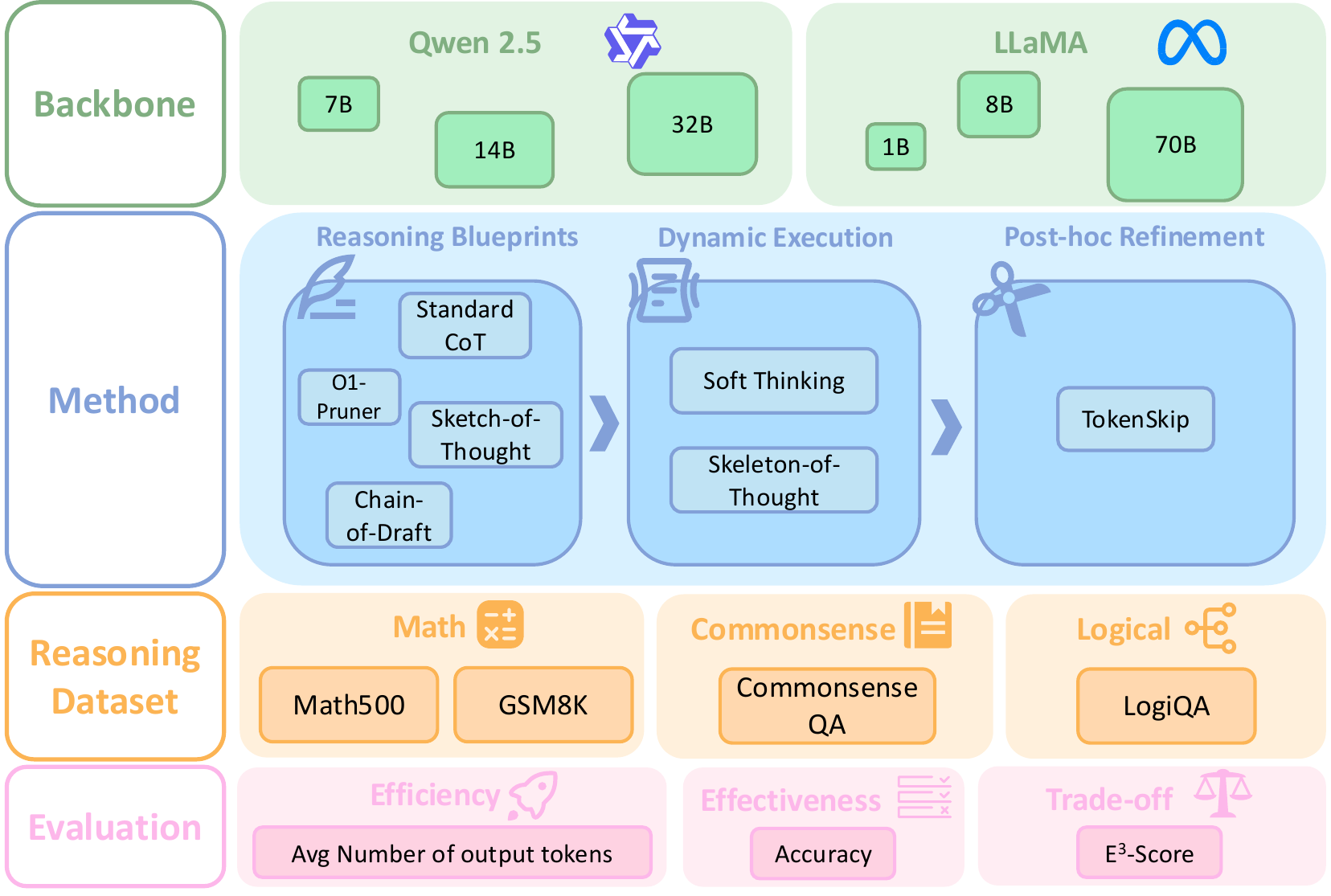}
    \caption{Overview of EffiReason-Bench, which compares efficient reasoning methods across diverse paradigms, backbones, and datasets.}
    \label{benchmark}
\end{figure}

\noindent Our contributions can be summarized as follows:

\ding{182} We propose \textbf{EffiReason-Bench}, the first benchmark to enable systematic and fair comparison of efficient reasoning methods across different paradigms, backbone scales, and reasoning domains.

\ding{183} We introduce the \textbf{E$^3$-Score}, a principled metric that provides smooth and reliable evaluation of efficiency–effectiveness trade-offs, addressing the limitations of prior measures.

\ding{184} We conduct extensive experiments with 7 methods, 6 backbones, and 4 datasets, revealing that optimal efficiency strategies depend on both backbones size and task complexity, and we will release all implementations for reproducibility.

\section{Related Work}

\noindent\textbf{The Overthinking Phenomenon}
The CoT paradigm~\cite{wei2022chain} enables LLMs to generate explicit intermediate reasoning steps, which has led to significant gains on complex reasoning tasks. However, subsequent studies have identified the tendency of models to produce excessively long and redundant reasoning chains, even for trivial problems~\cite{han-etal-2025-token,shen2025codi,xu2025softcot}. This ``overthinking'' increases computational cost and inference latency, and in some cases reduces accuracy due to the accumulation of errors. These observations have motivated a line of work on improving the efficiency of reasoning.

\noindent\textbf{Paradigms of Efficient Reasoning} Research on efficient reasoning can be broadly grouped into three paradigms depending on when efficiency interventions are applied. \textbf{Reasoning blueprints} modify the model’s behavior before inference, for example by fine-tuning with length-sensitive rewards~\cite{shen2025dast,luo2025o1,team2025kimi,yeo2025demystifying} or by designing concise prompts~\cite{xu2025chain,wang2025sampling,zhang2025lightthinker}. \textbf{Dynamic execution} methods adjust the reasoning process during inference, such as reasoning in latent space~\cite{hao2024training,cheng2024compressed,shen2025codi,shen2025efficient,xu2025softcot,zhang2025soft} or restructuring decoding with intermediate skeletons~\cite{ning2024skeleton}. \textbf{Post-hoc refinement} techniques operate after generation, e.g. pruning redundant steps~\cite{xia2025tokenskip}. These lines of work highlight complementary approaches to reducing reasoning cost, yet their effectiveness varies substantially across models and tasks.

\noindent\textbf{Benchmarks for Reasoning Efficiency}
Several benchmarks have been proposed to evaluate reasoning or efficiency, but they are limited in scope. Sys2Bench~\cite{parashar2025inference} provides systematic evaluation for inference-time search strategies, while other studies focus on model-level compression such as quantization and pruning~\cite{zhang2025reasoning,liu2025quantization}. Additional efforts investigate isolated phenomena, including the effect of reasoning length~\cite{jin2024impact}, overthinking in agentic settings~\cite{cuadron2025danger}, or latent-space reasoning~\cite{hagendorff2025beyond}. Although valuable, these works remain siloed, making it difficult to compare across paradigms or to establish general principles for efficient reasoning. Nonetheless, even with available benchmarks, achieving a fair comparison of efficiency–effectiveness balance crucially depends on the choice of evaluation metric. Prior measures exhibit clear limitations. \textit{Accuracy per Computation Unit (ACU)}~\cite{ma2025cot} normalizes accuracy by computational cost, but it undervalues small yet meaningful gains in high-accuracy regimes, where marginal improvements are most difficult to achieve. \textit{Accuracy–Efficiency Score (AES)}~\cite{luo2025o1} combines accuracy and efficiency through a piecewise formulation, but it relies on a series of manually chosen thresholds and heuristic hyperparameters, making results sensitive and often unstable across settings.

Our study differs from existing efforts in two key aspects. First, we establish \textbf{EffiReason-Bench}, the first benchmark that systematically compares efficient reasoning methods across paradigms, model scales, and reasoning domains under a unified experimental setup. Second, we propose the \textbf{E$^3$-Score}, a principled metric that provides smooth and fair evaluation of efficiency–effectiveness trade-offs, addressing the limitations of prior measures. Together, these contributions offer a rigorous and reproducible foundation for advancing research on efficient reasoning.

\section{Dataset, Task, and Pipeline Setup}
\textbf{EffiReason-Bench} addresses three key limitations in current evaluation practices. First, heterogeneous backbones and training strategies across prior studies make it challenging to isolate the specific contributions of efficiency methods from inherent LLM capabilities. Second, cross-domain generalization remains insufficiently explored, as most evaluations concentrate on specific task families. Third, existing metrics measuring the efficiency–effectiveness trade-off, suffer from discontinuities and heuristic dependencies, limiting their reliability. To address these challenges, EffiReason-Bench provides a unified evaluation platform enabling systematic \textit{cross-paradigm}, \textit{cross-backbone}, and \textit{cross-domain} assessment.

\subsection{Datasets}
To ensure coverage of diverse reasoning types, we select four established datasets spanning mathematical, commonsense, and logical reasoning:
\textbf{GSM8K}~\cite{cobbe2021training}, a collection of grade school math word problems requiring multi-step arithmetic;
\textbf{MATH500}, a curated subset of the MATH dataset~\cite{hendrycks2021measuring} comprising 500 competition-level problems designed to balance difficulty with computational feasibility for large-scale benchmarking;
\textbf{CommonsenseQA}~\cite{talmor2018commonsenseqa}, a multiple-choice dataset evaluating everyday commonsense reasoning;
and \textbf{LogiQA}~\cite{liu2020logiqa}, a benchmark derived from national logic examinations requiring conditional and deductive reasoning. 

Since CommonsenseQA and LogiQA provide only final answers without explanatory reasoning, we construct verified CoT solutions to enable comprehensive evaluation of step-by-step reasoning. To address the logical gaps and contradictions against ground-truth commonly found in auto-generated CoT data, we propose a rigorous construction pipeline with three core components. \textbf{First}, we adopt a standardized four-stage reasoning structure: premise identification, relationship formalization, condition evaluation, and option mapping. \textbf{Second}, we enforce explicit comparative analysis across all options: each solution must explain why the correct answer holds and why distractors fail within a unified logical framework, ensuring strict alignment between reasoning chains and ground-truth answers. \textbf{Third}, all annotations undergo dual human verification to eliminate logical inconsistencies and unsupported inferences. This pipeline yields high-quality CoT annotations that enable transparent cross-domain evaluation of reasoning processes beyond answer accuracy alone.

\subsection{Efficient Reasoning Methods}
We include 7 representative methods that span the three major paradigms of efficient reasoning, with Standard CoT serving as the reference substance.  

\noindent\textbf{Reasoning blueprints (Pre-process).}  
\emph{O1-Pruner} fine-tunes models using a length-aware reward function, encouraging concise reasoning while preserving correctness.  
\emph{Sketch-of-Thought (SoT)} introduces a lightweight router that maps inputs to high-level reasoning patterns (“sketches”), guiding the model to follow compressed reasoning templates.  
\emph{Chain-of-Draft} employs prompt instructions that restrict intermediate steps to very short drafts, explicitly constraining verbosity.  

\noindent\textbf{Dynamic execution (In-process).}  
\emph{Soft Thinking} enables reasoning in a continuous space by generating “concept tokens,” which represent probability-weighted mixtures of token embeddings. 
\emph{Skeleton-of-Thought (SloT)} restructures decoding into a two-stage process: first generating a high-level outline (skeleton) and then expanding details for each component, potentially reducing redundancy through parallel generation.  

\noindent\textbf{Post-hoc refinement (Post-process).}  
\emph{TokenSkip} operates on completed reasoning chains by removing tokens with low semantic contribution to the final answer, producing shorter rationales.  

\subsection{Backbone Models}
All methods are evaluated on 6 widely used open-source LLM backbones to ensure reproducibility and fair comparison across scales.  
We adopt the Qwen2.5-Instruct series (7B, 14B, 32B) and the LLaMA-Instruct series (1B, 8B, 70B), which enables controlled analysis of how efficiency strategies interact with model capacity.

\subsection{Training Criteria}
\subsubsection{Main Experiments}
We consider two regimes. In the \textit{Train-Free} setting, models perform direct inference without any parameter updates. In the \textit{Train-based} setting, models are first trained on the full training set and then evaluated by inference. 
For all datasets and backbones, the evaluation prompt consists of a fixed instruction header followed by the test instance. Answer formatting is standardized: the prediction must appear on a single line beginning with \texttt{``Final Answer:''}; multiple-choice tasks require a single option letter (\texttt{A/B/C/D}); open-ended math expects a numeric or short span. This unified protocol is consistently applied across all methods and both training regimes to ensure fair comparison.

\subsubsection{Few-shot Setting.}
Few-shot analysis is reported only when explicitly stated. We evaluate $k\in\{1,4,8,12\}$ shots. 
\noindent\emph{Exemplar construction and reuse.} Few-shot exemplars are drafted by GPT-4o and human-audited for clarity and correctness. For each dataset and each $k$, we finalize exemplar lists once with a fixed global seed and reuse them verbatim across methods, backbones, and both regimes to eliminate sampling variance. Each exemplar follows \texttt{Question $\rightarrow$ Reasoning (numbered steps) $\rightarrow$ Final Answer}. For mathematical tasks, the reasoning enumerates key intermediate arithmetic/transformations; for commonsense/logical tasks, it adopts \texttt{[Premises] $\rightarrow$ [Inference rule] $\rightarrow$ [Conclusion]} and, for multiple-choice data, includes a brief justification for distractors. Answer formatting follows the same rule as in the main experiments.

\noindent\emph{Variants.} Unless otherwise noted, exemplars are in-domain (same dataset as the evaluation task). We also probe a cross-domain variant (exemplars from a source dataset, evaluation on a target dataset from a different reasoning domain). To assess robustness, we introduce noisy exemplars via controlled perturbations covering: (i) reasoning-step noise (deletion/permutation/repetition), (ii) semantic noise (e.g., mild numeric or connector perturbations). Details are provided in Appendix~\ref{noise}.

\begin{table*}[htp]
\centering
\caption{Accuracy (\%) and average output tokens comparison of reasoning methods on \emph{GSM8K} dataset. Best performance is indicated in bold, and runner-up is underlined. Train-free methods are highlighted in green, while train-based methods are highlighted in red. ``OOM" is the abbreviation for out-of-memory \protect\footnotemark.}
\label{tab:gsm8k_performance}
\adjustbox{width=\textwidth}{
\begin{tabular}{@{}llcccccccccccccccccccccccc@{}}
\toprule
\textbf{Category} & \textbf{Methods} 
& \multicolumn{9}{c}{\textbf{Qwen 2.5 (Instruct)}} 
& \multicolumn{9}{c}{\textbf{LLaMA(Instruct)}} \\
\cmidrule(lr){3-11} \cmidrule(lr){12-20}
 & & \multicolumn{3}{c}{\textbf{7B}} & \multicolumn{3}{c}{\textbf{14B}} & \multicolumn{3}{c}{\textbf{32B}} 
 & \multicolumn{3}{c}{\textbf{1B}} & \multicolumn{3}{c}{\textbf{8B}} & \multicolumn{3}{c}{\textbf{70B}} \\
\cmidrule(lr){3-5} \cmidrule(lr){6-8} \cmidrule(lr){9-11} \cmidrule(lr){12-14} \cmidrule(lr){15-17} \cmidrule(lr){18-20}
 & & \textbf{Acc} $\uparrow$ & \textbf{Tokens} $\downarrow$ & \textbf{E$^3$} $\uparrow$
 & \textbf{Acc} $\uparrow$ & \textbf{Tokens} $\downarrow$ & \textbf{E$^3$} $\uparrow$
 & \textbf{Acc} $\uparrow$ & \textbf{Tokens} $\downarrow$ & \textbf{E$^3$} $\uparrow$
 & \textbf{Acc} $\uparrow$ & \textbf{Tokens} $\downarrow$ & \textbf{E$^3$} $\uparrow$
 & \textbf{Acc} $\uparrow$ & \textbf{Tokens} $\downarrow$ & \textbf{E$^3$} $\uparrow$
 & \textbf{Acc} $\uparrow$ & \textbf{Tokens} $\downarrow$ & \textbf{E$^3$} $\uparrow$ \\
\midrule
\textit{Reference Substance} & \cellcolor{lightgreen}CoT 
& \underline{90.60} & 292.86 & -
& \underline{93.63} & 278.61 & -
& 94.39 & 277.36 & -
& 22.44 & 227.37 & -
& 78.70 & 271.57 & -
& \textbf{95.68} & 240.65 & - \\
\midrule
\textit{Reasoning Blueprints} & \cellcolor{lightgreen}CoD 
& 64.82 & \textbf{45.94} & 0.21
& 79.23 & \textbf{48.67} & 0.28
& 86.35 & \textbf{56.09} & 0.40
& 1.82 & \textbf{138.96} & 0.25
& 58.45 & \underline{140.66} & 0.46
& 88.63 & \textbf{80.98} & 0.37 \\
 & \cellcolor{lightgreen}SoT 
& 72.02 & \underline{103.81} & 0.29
& 89.46 & \underline{89.28} & 0.61
& 89.69 & \underline{101.81} & 0.54
& 12.89 & 300.46 & 0.68
& 61.03 & \textbf{125.14} & 0.51
& 92.42 & \underline{84.81} & 0.57 \\
 & \cellcolor{lightred}O1-Pruner 
& 79.15 & 206.12 & 0.42
& 86.58 & 233.79 & 0.46
& OOM & OOM & OOM
& \underline{36.39} & 170.94 & \underline{1.44}
& 73.24 & 231.61 & 0.80
& OOM & OOM & OOM \\
\midrule
\textit{Dynamic Execution} & \cellcolor{lightgreen}Soft Thinking 
& \textbf{90.67} & 287.09 & \textbf{1.01}
& \textbf{94.39} & 272.54 & \textbf{1.14}
& \underline{94.47} & 276.39 & \underline{1.01}
& 35.10 & 208.84 & 1.20
& \textbf{84.31} & 248.25 & \textbf{1.36}
& \underline{95.53} & 238.12 & \underline{0.97} \\
 & \cellcolor{lightgreen}SloT 
& 70.05 & 109.78 & 0.27
& 85.67 & 134.29 & 0.43
& 86.88 & 127.25 & 0.41
& 4.55 & 209.63 & 0.48
& 53.68 & 324.74 & 0.36
& 93.86 & 206.63 & 0.70 \\
\midrule
\textit{Post-hoc Refinement} & \cellcolor{lightred}TokenSkip 
& 88.63 & 214.91 & \underline{0.84}
& 93.10 & 219.10 & \underline{0.93}
& \textbf{95.38} & 224.93 & \textbf{1.23}
& \textbf{43.67} & \underline{165.48} & \textbf{1.54}
& \underline{81.88} & 149.88 & \underline{1.31}
& \textbf{95.68} & 172.69 & \textbf{1.01} \\
\bottomrule
\end{tabular}}
\vspace{-0.5em}
\end{table*}
\footnotetext{OOM errors occurred during O1-Pruner's required, memory-intensive fine-tuning phase, which exceeded the 4x 80G A800 capacity for 32B/70B models, even though we set batch size = 1}

\subsection{Evaluation Strategy}
We evaluate models on both \emph{effectiveness} (accuracy) and \emph{efficiency} (average output tokens). To evaluate the efficiency-effectiveness trade-off, several prior metrics attempt to combine these two dimensions, but each has limitations. \textbf{ACU} linearly scales accuracy by computation, undervaluing small but crucial improvements in high-accuracy regimes. 
\textbf{AES} considers relative gains in accuracy and tokens, but it is piecewise and relies on heuristic hyperparameters, leading to instability. 

To overcome these issues, we propose the \textbf{Efficiency--Effectiveness Equilibrium Score (E$^3$-Score)}, a smooth metric inspired by the Constant Elasticity of Substitution (CES) formulation~\cite{mcfadden1963constant}, which has long been used to model trade-offs with controlled substitutability. 
Given baseline accuracy $A_0$ and tokens $T_0$, and method values $A, T$, we define
\begin{equation}
r_{\text{acc}} = (\tfrac{A}{1-A})/(\tfrac{A_0}{1-A_0}), 
\quad
r_{\text{tok}} = \frac{T_0}{T}.
\end{equation}
The E$^3$-Score is then given by
\begin{equation}
\small
\mathrm{E^3}(A,T) = \Big( w \cdot r_{\text{acc}}^{\rho} + (1-w)\cdot r_{\text{tok}}^{\rho} \Big)^{\tfrac{1}{\rho}},
\end{equation}
where the weight $w$ calibrates the \emph{relative importance} of accuracy versus efficiency to task difficulty; setting $w=A_0$ (the baseline accuracy) makes accuracy more prominent when the baseline is already strong and avoids extra tuning. The parameter $\rho$ governs \emph{substitutability}: with $\rho<0$ the two factors behave as complements so that efficiency gains cannot fully offset accuracy losses. We adopt $\rho=-1$ (a weighted harmonic mean) as the default, providing a conservative balance that penalizes asymmetric improvements and prevents excessive token savings from overshadowing accuracy degradation; without this knob (e.g., implicitly fixing $\rho=0$ or $1$) the metric becomes too permissive and less reliable.
This CES-inspired aggregation ensures a smooth trade-off: accuracy gains at high baselines are emphasized, efficiency is rewarded but cannot fully offset accuracy drops, and two-way degradation is strongly penalized. 

\begin{theorem}[Global sensitivity of $E_3$ to $\rho$]
\label{thm:rho-sensitivity}
Let $a=r_{\text{acc}}>0$, $b=r_{\text{tok}}>0$, and $w\in(0,1)$. Define $\Delta\coloneqq\lvert\log(a/b)\rvert$.
For any $\rho_1,\rho_2\in\mathbb{R}$ (interpreting $\rho=0$ by continuity of the power mean),
\begin{equation}
\small
    \bigl|\log E_3(A,T;\rho_1)-\log E_3(A,T;\rho_2)\bigr|
\ \le\ \frac{\Delta^2}{8}\,|\rho_1-\rho_2|.
\end{equation}
Equivalently, $\bigl|\partial_\rho \log E_3(A,T;\rho)\bigr|\le \Delta^2/8$ for all $\rho$, and
\begin{equation}
\small
    e^{-\frac{\Delta^2}{8}\,|\rho_1-\rho_2|}
\ \le\ \frac{E_3(A,T;\rho_1)}{E_3(A,T;\rho_2)}
\ \le\ e^{\frac{\Delta^2}{8}\,|\rho_1-\rho_2|}.
\end{equation}
The bound is independent of $w$.
\end{theorem}

\section{Experiments and Analysis}

\subsection{Mathematical Reasoning Tasks}
We evaluate performance on two mathematical reasoning datasets: GSM8K and MATH500. Tables 1 and 2 reveal that efficiency strategy effectiveness is highly contingent on task complexity, backbone scale, and architectural choices.

\begin{table*}[htp]
\centering
\caption{Accuracy (\%) and average output tokens comparison of reasoning methods on \emph{MATH 500} dataset.}
\label{tab:math500_performance}
\adjustbox{width=\textwidth}{
\begin{tabular}{@{}llcccccccccccccccccccccccc@{}}
\toprule
\textbf{Category} & \textbf{Methods} 
& \multicolumn{9}{c}{\textbf{Qwen 2.5 (Instruct)}} 
& \multicolumn{9}{c}{\textbf{LLaMA (Instruct)}} \\
\cmidrule(lr){3-11} \cmidrule(lr){12-20}
 & & \multicolumn{3}{c}{\textbf{7B}} & \multicolumn{3}{c}{\textbf{14B}} & \multicolumn{3}{c}{\textbf{32B}} 
 & \multicolumn{3}{c}{\textbf{1B}} & \multicolumn{3}{c}{\textbf{8B}} & \multicolumn{3}{c}{\textbf{70B}} \\
\cmidrule(lr){3-5} \cmidrule(lr){6-8} \cmidrule(lr){9-11} \cmidrule(lr){12-14} \cmidrule(lr){15-17} \cmidrule(lr){18-20}
 & & \textbf{Acc} $\uparrow$ & \textbf{Tokens} $\downarrow$ & \textbf{E$^3$} $\uparrow$
 & \textbf{Acc} $\uparrow$ & \textbf{Tokens} $\downarrow$ & \textbf{E$^3$} $\uparrow$
 & \textbf{Acc} $\uparrow$ & \textbf{Tokens} $\downarrow$ & \textbf{E$^3$} $\uparrow$
 & \textbf{Acc} $\uparrow$ & \textbf{Tokens} $\downarrow$ & \textbf{E$^3$} $\uparrow$
 & \textbf{Acc} $\uparrow$ & \textbf{Tokens} $\downarrow$ & \textbf{E$^3$} $\uparrow$
 & \textbf{Acc} $\uparrow$ & \textbf{Tokens} $\downarrow$ & \textbf{E$^3$} $\uparrow$ \\
\midrule
\textit{Reference Substance} & \cellcolor{lightgreen}CoT 
& \underline{75.00} & 579.00 & --
& \underline{77.00} & 558.31 & --
& \textbf{81.40} & 544.64 & --
& 13.00 & 701.48 & --
& 33.20 & 739.55 & --
& \textbf{75.20} & 582.59 & -- \\
\midrule
\textit{Reasoning Blueprints} & \cellcolor{lightgreen}CoD 
& 14.80 & \textbf{145.34} & 0.08
& 45.00 & \textbf{71.31} & 0.31
& 55.40 & \textbf{93.40} & 0.34
& 2.80 & \textbf{246.15} & 1.02
& 23.40 & \textbf{363.00} & \textbf{1.15}
& 50.00 & \textbf{249.99} & 0.42 \\
 & \cellcolor{lightgreen}SoT 
& 56.80 & 328.58 & 0.54
& 62.60 & \underline{177.73} & \underline{0.62}
& 68.60 & 192.79 & \underline{0.59}
& 8.80 & 556.54 & 1.12
& 26.20 & 486.41 & \underline{1.11}
& 64.80 & \underline{277.70} & 0.74 \\
 & \cellcolor{lightred}O1-Pruner 
& 59.20 & 455.38 & \underline{0.57}
& 64.00 & 514.32 & 0.60
& OOM & OOM & OOM
& 17.40 & \underline{364.58} & \textbf{1.84}
& 37.00 & 464.14 & 1.43
& OOM & OOM & OOM \\
\midrule
\textit{Dynamic Execution} & \cellcolor{lightgreen}Soft Thinking 
& \textbf{75.40} & 612.63 & \textbf{1.00}
& \textbf{80.20} & 582.17 & \textbf{1.14}
& \underline{81.00} & 555.23 & \textbf{0.98}
& \underline{24.00} & 704.77 & 1.07
& \textbf{48.00} & 660.48 & 1.29
& \underline{73.60} & 555.74 & \textbf{0.95} \\
 & \cellcolor{lightgreen}SloT 
& 46.20 & \underline{157.77} & 0.37
& 57.00 & 184.93 & 0.49
& 58.80 & \underline{172.56} & 0.39
& 5.40 & 579.19 & 0.94
& 22.60 & 587.14 & 0.91
& 59.00 & 312.02 & 0.58 \\
\midrule
\textit{Post-hoc Refinement} & \cellcolor{lightred}TokenSkip 
& 40.80 & 314.40 & 0.29
& 42.20 & 318.73 & 0.27
& 44.40 & 314.16 & 0.22
& \textbf{27.60} & 404.98 & \underline{1.81}
& \underline{46.00} & \underline{355.32} & 1.94
& 66.60 & 372.25 & \underline{0.77} \\
\bottomrule
\end{tabular}}
\vspace{-0.5em}
\end{table*}
Within the \textbf{Reasoning Blueprints category}, we observe a fundamental accuracy-efficiency trade-off. The train-based O1-Pruner consistently prioritizes accuracy preservation, often surpassing the CoT baseline on MATH500 for LLaMA 1B and 8B, while achieving more modest compression ratios. Conversely, train-free methods such as CoD and SoT attain maximum token compression but at significant accuracy cost, demonstrating the inverse tendency. This pattern suggests that learned pruning strategies better preserve reasoning integrity than heuristic compression approaches.

The \textbf{Dynamic Execution} paradigm exhibits striking performance divergence. Soft Thinking maintains near-identical accuracy to the CoT baseline on GSM8K and consistently matches or exceeds it on MATH500, demonstrating remarkable robustness across backbone scales. In contrast, Skeleton-of-Thought induces substantial accuracy degradation in most configurations, particularly on the more challenging MATH500 benchmark. This disparity indicates that adaptive token generation mechanisms prove more reliable than parallel decoding strategies for complex reasoning tasks.

\textbf{Post-hoc Refinement} through TokenSkip reveals a critical backbone-dependent phenomenon. While the method achieves balanced accuracy-efficiency trade-offs on GSM8K across most of the backbones, its behavior on MATH500 diverges dramatically by architecture. TokenSkip induces catastrophic accuracy collapse on Qwen, with degradations ranging from 34.2\% to 37.0 \%, yet proves remarkably effective on LLaMA backbones, improving accuracy on 1B and 8B variants while incurring only an 8.6 \% loss on the 70B variants. This architectural sensitivity suggests that TokenSkip's pruning criteria align with LLaMA's internal reasoning representations but fundamentally misalign with Qwen's computational structure, highlighting the importance of architecture-aware optimization.

\subsection{Commonsense Reasoning Tasks}
On the CommonsenseQA dataset (Table 3), we observe a stark contrast to the mathematical reasoning tasks. Commonsense reasoning demonstrates high robustness to token compression, with many efficiency methods maintaining stable or even superior accuracy while drastically reducing token counts.
\begin{table*}[htp]
\centering
\caption{Accuracy (\%) and average output tokens comparison of reasoning methods on \emph{Commonsense} dataset.}
\label{tab:commonse_performance}
\adjustbox{width=\textwidth}{
\begin{tabular}{@{}llcccccccccccccccccccccccc@{}}
\toprule
\textbf{Category} & \textbf{Methods} 
& \multicolumn{9}{c}{\textbf{Qwen 2.5 (Instruct)}} 
& \multicolumn{9}{c}{\textbf{LLaMA (Instruct)}} \\
\cmidrule(lr){3-11} \cmidrule(lr){12-20}
 & & \multicolumn{3}{c}{\textbf{7B}} & \multicolumn{3}{c}{\textbf{14B}} & \multicolumn{3}{c}{\textbf{32B}} 
 & \multicolumn{3}{c}{\textbf{1B}} & \multicolumn{3}{c}{\textbf{8B}} & \multicolumn{3}{c}{\textbf{70B}} \\
\cmidrule(lr){3-5} \cmidrule(lr){6-8} \cmidrule(lr){9-11} \cmidrule(lr){12-14} \cmidrule(lr){15-17} \cmidrule(lr){18-20}
 & & \textbf{Acc} $\uparrow$ & \textbf{Tokens} $\downarrow$ & \textbf{E$^3$} $\uparrow$
 & \textbf{Acc} $\uparrow$ & \textbf{Tokens} $\downarrow$ & \textbf{E$^3$} $\uparrow$
 & \textbf{Acc} $\uparrow$ & \textbf{Tokens} $\downarrow$ & \textbf{E$^3$} $\uparrow$
 & \textbf{Acc} $\uparrow$ & \textbf{Tokens} $\downarrow$ & \textbf{E$^3$} $\uparrow$
 & \textbf{Acc} $\uparrow$ & \textbf{Tokens} $\downarrow$ & \textbf{E$^3$} $\uparrow$
 & \textbf{Acc} $\uparrow$ & \textbf{Tokens} $\downarrow$ & \textbf{E$^3$} $\uparrow$ \\
\midrule
\textit{Reference Substance} & \cellcolor{lightgreen}CoT 
& \underline{80.00} & 286.47 & -
& 81.23 & 204.10 & -
& 84.20 & 220.43 & -
& 18.52 & 216.88 & -
& 62.72 & 249.77 & -
& \underline{83.37} & 332.19 & - \\
\midrule
\textit{Reasoning Blueprints} & \cellcolor{lightgreen}CoD 
& \textbf{82.47} & \textbf{7.59} & \textbf{1.46}
& \underline{81.40} & \textbf{7.58} & \textbf{1.23}
& 83.79 & \textbf{10.06} & 1.14
& 4.61 & \textbf{117.40} & 0.76
& 56.38 & \textbf{56.25} & 1.11
& 80.33 & \textbf{10.14} & 0.97 \\
 & \cellcolor{lightgreen}SoT 
& 78.35 & \underline{15.72} & \underline{1.12}
& 80.99 & \underline{26.62} & \underline{1.18}
& \underline{84.53} & \underline{28.08} & \textbf{1.19}
& 12.84 & 245.68 & 0.83
& 55.31 & 179.47 & 0.89
& 80.25 & \underline{36.27} & \textbf{1.14} \\
 & \cellcolor{lightred}O1-Pruner 
& 78.68 & 117.94 & 1.05
& 80.66 & 198.85 & 0.97
& OOM & OOM & OOM
& \underline{39.75} & 184.61 & 1.32
& \underline{71.28} & 100.87 & \textbf{1.74}
& OOM & OOM & OOM \\
\midrule
\textit{Dynamic Execution} & \cellcolor{lightgreen}Soft Thinking 
& \underline{80.00} & 273.83 & 1.01
& 79.84 & 180.40 & 0.95
& 83.62 & 205.00 & 0.97
& 23.87 & 157.51 & \underline{1.38}
& 69.47 & 212.05 & 1.28
& \textbf{83.54} & 318.85 & 1.02 \\
 & \cellcolor{lightgreen}SloT 
& 72.35 & 48.13 & 0.80
& 77.04 & 119.52 & 0.86
& 76.95 & 132.36 & 0.69
& 4.28 & 163.31 & 0.64
& 25.50 & 236.54 & 0.29
& 79.51 & 180.51 & 0.86 \\
\midrule
\textit{Post-hoc Refinement} & \cellcolor{lightred}TokenSkip 
& \underline{79.34} & 209.53 & 1.02
& \textbf{81.98} & 147.99 & 1.10
& \textbf{85.60} & 145.66 & \underline{1.16}
& \textbf{53.25} & \underline{134.61} & \textbf{1.84}
& \textbf{73.58} & \underline{159.44} & \underline{1.62}
& 83.21 & 241.30 & \underline{1.04} \\
\bottomrule
\end{tabular}}
\vspace{-0.5em}
\end{table*}

\textbf{Reasoning Blueprints} exhibit extreme efficiency. The train-free CoD and SoT methods successfully compress token counts by over 90\% (e.g., Qwen 7B's 286.47 tokens reduced to 7.59 by CoD). Under this extreme compression, CoD largely maintains accuracy on Qwen (even slightly improving it on 7B and 14B) but shows a performance drop on LLaMA. The train-based O1-Pruner again demonstrates its ability to enhance weaker baselines, significantly improving LLaMA 8B's accuracy from 62.72\% to 71.28\%.

Among \textbf{Dynamic Execution} methods, Soft Thinking again proves its robustness. Its accuracy and token count remain highly consistent with the CoT baseline across all backbones (e.g., on Qwen 7B and LLaMA 70B) and deliver a notable improvement on LLaMA 8B (62.72\% to 69.47\%). Conversely, SloT performs poorly on this task. While it significantly reduces tokens, it also incurs a consistent accuracy drop across most of the backbones, culminating in a performance collapse on LLaMA 8B (62.72\% to 25.50\%).

\textbf{Post-hoc Refinement}: TokenSkip emerges as an exceptionally safe and effective strategy for commonsense reasoning. While moderately reducing tokens, it either maintains or improves accuracy on all backbones. It shows positive effects on both mid-performance baselines (like LLaMA 8B, from 62.72\% to 73.58\%) and high-performance ones (like Qwen 32B, from 84.20\% to 85.60\%). This performance stands in sharp contrast to its catastrophic impact on Qwen in the MATH500 task.
 
\subsection{Logical Reasoning Tasks}
Results on the LogiQA dataset presented in Table 4 reveal the heightened sensitivity of logical reasoning to compression strategies. In contrast to commonsense reasoning, logical inference demands greater preservation of reasoning step integrity, rendering aggressive compression approaches substantially more precarious.
\begin{table*}[htp]
\centering
\caption{Accuracy (\%) and average output tokens comparison of reasoning methods on \emph{LogiQA} dataset. }
\label{tab:logiqa_performance}
\adjustbox{width=\textwidth}{
\begin{tabular}{@{}llcccccccccccccccccccccccc@{}}
\toprule
\textbf{Category} & \textbf{Methods} 
& \multicolumn{9}{c}{\textbf{Qwen 2.5(Instruct)}} 
& \multicolumn{9}{c}{\textbf{LLaMA(Instruct)}} \\
\cmidrule(lr){3-11} \cmidrule(lr){12-20}
 & & \multicolumn{3}{c}{\textbf{7B}} & \multicolumn{3}{c}{\textbf{14B}} & \multicolumn{3}{c}{\textbf{32B}} 
 & \multicolumn{3}{c}{\textbf{1B}} & \multicolumn{3}{c}{\textbf{8B}} & \multicolumn{3}{c}{\textbf{70B}} \\
\cmidrule(lr){3-5} \cmidrule(lr){6-8} \cmidrule(lr){9-11} \cmidrule(lr){12-14} \cmidrule(lr){15-17} \cmidrule(lr){18-20}
 & & \textbf{Acc} $\uparrow$ & \textbf{Tokens} $\downarrow$ & \textbf{E$^3$} $\uparrow$
 & \textbf{Acc} $\uparrow$ & \textbf{Tokens} $\downarrow$ & \textbf{E$^3$} $\uparrow$
 & \textbf{Acc} $\uparrow$ & \textbf{Tokens} $\downarrow$ & \textbf{E$^3$} $\uparrow$
 & \textbf{Acc} $\uparrow$ & \textbf{Tokens} $\downarrow$ & \textbf{E$^3$} $\uparrow$
 & \textbf{Acc} $\uparrow$ & \textbf{Tokens} $\downarrow$ & \textbf{E$^3$} $\uparrow$
 & \textbf{Acc} $\uparrow$ & \textbf{Tokens} $\downarrow$ & \textbf{E$^3$} $\uparrow$ \\
\midrule
\textit{Reference Substance} & \cellcolor{lightgreen}CoT 
& 51.92 & 471.95 & --
& 56.37 & 422.58 & --
& 63.44 & 399.54 & --
& 14.29 & 461.17 & --
& 33.95 & 534.32 & --
& \underline{63.13} & 530.36 & -- \\
\midrule
\textit{Reasoning Blueprints} & \cellcolor{lightgreen}CoD 
& 49.77 & \textbf{27.45} & \underline{1.68}
& 55.76 & \textbf{20.35} & \textbf{1.67}
& 60.52 & \textbf{51.35} & \underline{1.31}
& 8.60 & \textbf{105.99} & \textbf{2.22}
& 30.11 & \underline{280.66} & 1.33
& 53.00 & \textbf{35.66} & 1.02 \\
 & \cellcolor{lightgreen}SoT 
& \underline{53.46} & \underline{83.58} & \textbf{1.74}
& 55.61 & \underline{85.50} & \underline{1.49}
& 63.90 & \underline{96.43} & \textbf{1.41}
& 11.98 & 673.31 & 0.70
& 32.26 & 510.20 & 1.00
& 58.22 & \underline{142.67} & \underline{1.14} \\
 & \cellcolor{lightred}O1-Pruner 
& 51.92 & 417.19 & 1.06
& 55.61 & 463.90 & 0.94
& OOM & OOM & OOM
& \underline{24.27} & 320.35 & 1.49
& \underline{43.16} & 303.59 & \underline{1.65}
& OOM & OOM & OOM \\
\midrule
\textit{Dynamic Execution} & \cellcolor{lightgreen}Soft Thinking 
& \textbf{55.15} & 475.83 & 1.06
& \textbf{60.37} & 405.13 & 1.12
& \underline{65.44} & 393.80 & 1.06
& 17.67 & 371.33 & 1.25
& \textbf{44.09} & 426.06 & 1.34
& \underline{63.13} & 506.40 & 1.02 \\
 & \cellcolor{lightgreen}SloT 
& 46.54 & 90.56 & 1.36
& 56.84 & 145.03 & 1.42
& 62.06 & 156.49 & 1.23
& 7.07 & \underline{231.50} & 1.35
& 25.50 & 376.36 & 1.03
& 58.06 & 238.99 & 1.06 \\
\midrule
\textit{Post-hoc Refinement} & \cellcolor{lightred}TokenSkip 
& 51.61 & 365.02 & 1.11
& \underline{58.53} & 305.96 & 1.20
& \textbf{66.05} & 294.23 & 1.20
& \textbf{27.19} & 262.90 & \underline{1.81}
& \textbf{44.09} & \textbf{278.12} & \textbf{1.77}
& \textbf{63.59} & 357.94 & \textbf{1.15} \\
\bottomrule
\end{tabular}}
\vspace{-0.5em}
\end{table*}
The \textbf{Reasoning Blueprints} category yields mixed results. The train-free CoD method, while achieving the most extreme token compression, consistently degrades accuracy across all backbones. SoT is similarly aggressive in compression but exhibits unstable performance: it marginally improves accuracy on Qwen 7B and 32B but harms performance on LLaMA backbones and other Qwen. In contrast, the train-based O1-Pruner again proves highly effective at improving weak backbones, substantially lifting LLaMA 1B accuracy from 14.29\% to 24.27\% and LLaMA 8B from 33.95\% to 43.16\%.

Within the \textbf{Dynamic Execution} category, Soft Thinking is a standout performer for accuracy, achieving performance gains across all backbones and notably on LLaMA 8B (33.95\% $\rightarrow$44.09\%) and Qwen 14B (56.37\% $\rightarrow$ 60.37\%). This accuracy enhancement, however, comes at the cost of minimal efficiency gains, as its token consumption is nearly identical to the CoT baseline. Conversely, SloT performs poorly on this task. Despite a neutral result on Qwen 14B, it degrades accuracy on almost all other backbones and causes a performance collapse on LLaMA 1B.

\textbf{Post-hoc Refinement} via TokenSkip emerges as a balanced and robust strategy for logical reasoning. It achieves moderate token compression while universally maintaining or improving accuracy across all backbones and scales. Its benefits are particularly pronounced on the LLaMA backbone, where it dramatically improves the LLaMA 1B (14.29\% $\rightarrow$ 27.19\%) and LLaMA 8B (33.95\% $\rightarrow$ 44.09\%) baselines, and it also delivers a strong gain on Qwen 32B (63.44\% $\rightarrow$ 66.05\%).

\subsection{Few-shot Learning Setting}
\label{sec:few-shot}
In the few-shot experiments, we compare four representative methods: SloT, SoT, CoT, and CoD. Other approaches, such as TokenSkip and O1-Pruner, were excluded, as they did not learn effective patterns under limited supervision. Results are reported under 1-, 4-, 8-, and 12-shot conditions.  

\paragraph{Clean few-shot evaluation.}
Key obversions are shown as follows: \textbf{First}, as illustrated by the solid lines, CoT serves as a high-accuracy, high-token baseline across most of the tasks.
\textbf{Second}, the most significant finding is the superior scalability of SloT (solid green line). On MATH500, GSM8K, and CommonsenQA, SloT's accuracy improves with more shots, even approaching CoT. On Commonsense, SloT's accuracy surpasses that of CoT at 4, 8, or 12 shots. A potential mechanism is that SloT's structured nature (skeleton-then-expansion) allows it to effectively learn the structural patterns of reasoning from exemplars; thus, more shots lead to a more robust learned reasoning structure.
\textbf{Third}, conversely, the Reasoning Blueprint methods (CoD and SoT) exhibit rigidity. While CoD (solid orange line) and SoT (solid red line) achieve extreme token compression, their accuracy shows more instability than CoT and SloT across all tasks and shows minimal improvement on GSM8K with more shots. This suggests these methods primarily learn a surface-level pattern ("be short"), and this rigid compression inhibits their ability to learn complex reasoning from additional exemplars.
\begin{figure*}[htp]
    \centering
    \includegraphics[width=0.7\linewidth]{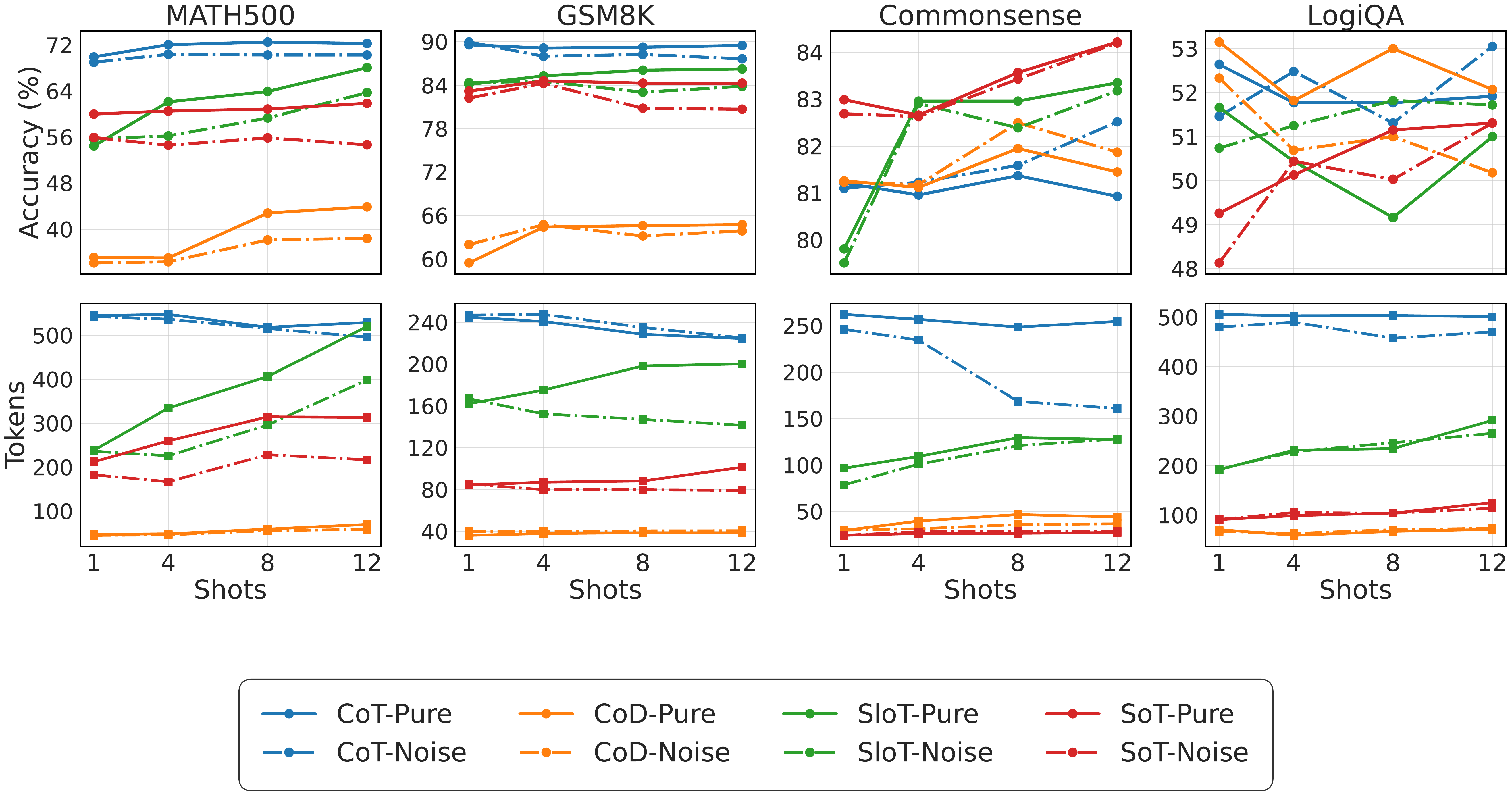}
    \caption{Few-shot performance (accuracy on top, tokens on bottom) of four reasoning methods on Qwen2.5-7B across four datasets, comparing clean vs. noisy settings.}
    \label{few-shot}
\end{figure*}

\paragraph{Noise injection setting.}
We further examine robustness by injecting noise into few-shot exemplars.  
In most cases, test methods degrade under noisy demonstrations, but the extent differs. SoT and CoD are most vulnerable, especially on mathematically intensive tasks, while SloT shows comparatively stable performance. CoT remains the most consistent across domains, with minimal losses and sometimes even shorter outputs. This suggests that blueprint methods depend heavily on surface-level exemplar patterns, so perturbations in reasoning steps, semantics, or formatting directly disrupt their effectiveness. In contrast, SloT extracts higher-level structural organization, which is less sensitive to local corruption, and CoT’s unconstrained generation helps maintain baseline stability.  
\textit{Implication.} In practice, exemplar quality cannot always be guaranteed. Therefore, efficiency methods must consider robustness: structure-aware decoding (SloT) offers a safer choice under noisy supervision, while blueprint compression strategies need additional safeguards to remain reliable.

\subsection{Cross-domain Reasoning Setting}
We further evaluate cross-domain transfer in the few-shot setting, where exemplars are drawn from a source domain and evaluated on a target task from a different domain. The results, shown in Figure 3 and Figure 4, indicate that cross-domain transfer presents a significant robustness challenge, with performance varying considerably across methods, transfer paths, and shot counts.
\begin{figure}[htp]
    \centering
    \includegraphics[width=1.0\linewidth]{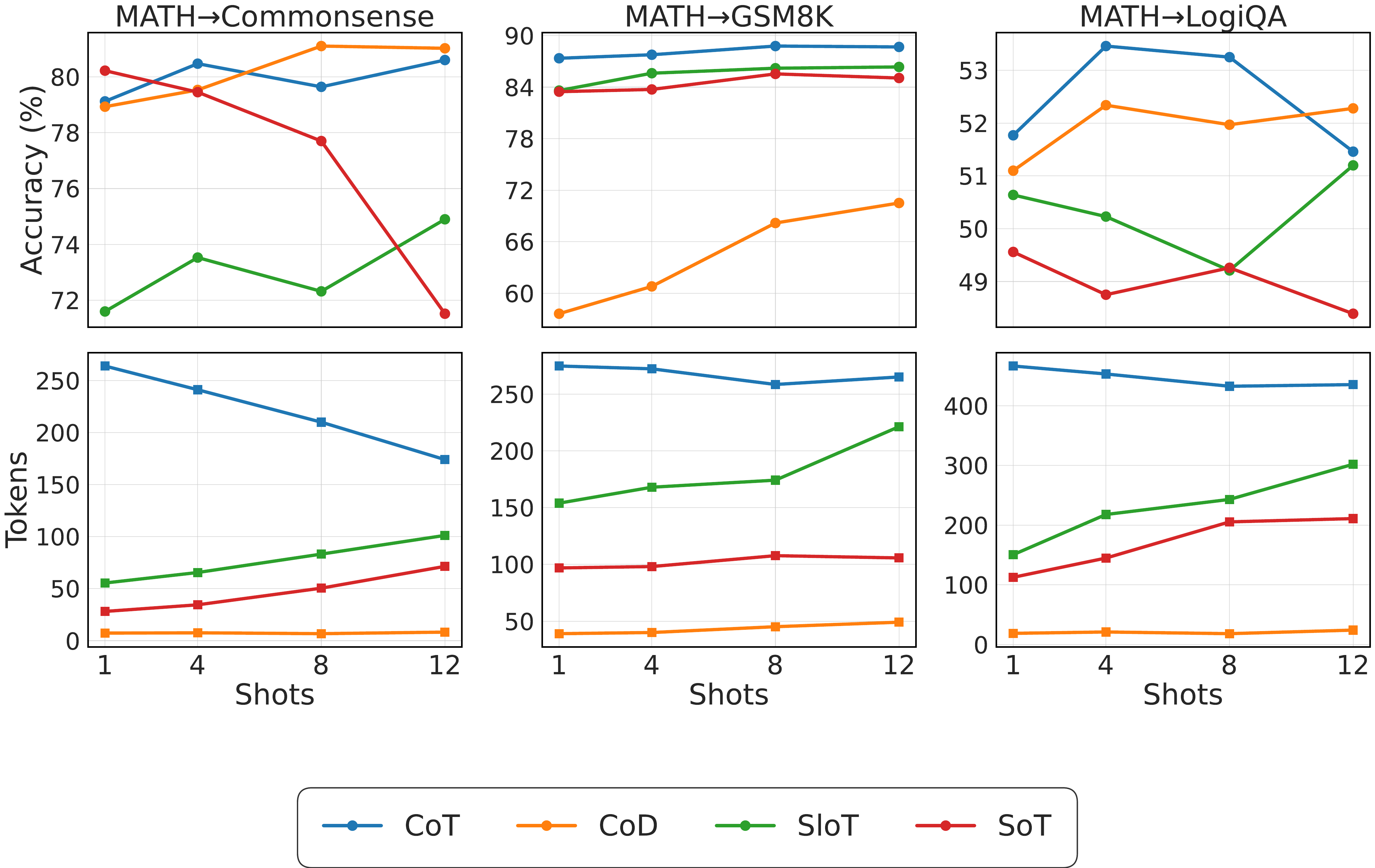}
    \caption{Comparison of few-shot transfer from \textbf{MATH500} to Commonsense, GSM8K, and LogiQA, in terms of accuracy (top) and tokens (bottom).}
    \label{fig:placeholder}
\end{figure}

\textbf{First}, CoT generally demonstrates the strongest stability. In most transfer scenarios, CoT (blue line) maintains a high level of accuracy that is relatively less perturbed by the number of shots or the change in domain, serving as a robust cross-domain baseline.
\textbf{Second}, SloT shows shot-dependent cross-domain reasoning transfer. SloT's (green line) performance is not always strong at 1-shot, as seen in the LOGIQA->MATH transfer, where it starts far below CoT. However, SloT is highly sensitive to exemplar count: its accuracy rises substantially with more shots in multiple settings (e.g., LOGIQA->MATH and MATH->Commonsense). This suggests SloT may be capable of learning transferable reasoning structures from out-of-domain exemplars, but this learning is more data-dependent than CoT's.
\textbf{Third}, Reasoning Blueprint methods (CoD, SoT) transfer poorly and unreliably. CoD (orange line) exhibits low accuracy in most cross-domain settings, though it shows modest improvement with more shots in near-domain transfer (e.g., MATH->GSM8K). SoT's (red line) behavior is particularly volatile: in the MATH->Commonsense transfer, its accuracy plummets as shots increase, while in the LOGIQA->GSM8K transfer, it remains consistently low.
\textbf{Forth}, domain similarity appears to be a key factor. Performance in near-domain transfers (e.g., MATH->GSM8K, both math tasks) is generally more stable for most methods than in far-domain transfers (e.g., LOGIQA->Commonsense). In most of the far-domain scenarios, we observe more significant volatility in accuracy for nearly all methods, including CoT and SloT. This indicates that methods reliant on surface patterns (CoD, SoT) or specific structures (SloT) may be particularly brittle when faced with a substantial domain shift.
\begin{figure}[htp]
    \centering
    \includegraphics[width=1.0\linewidth]{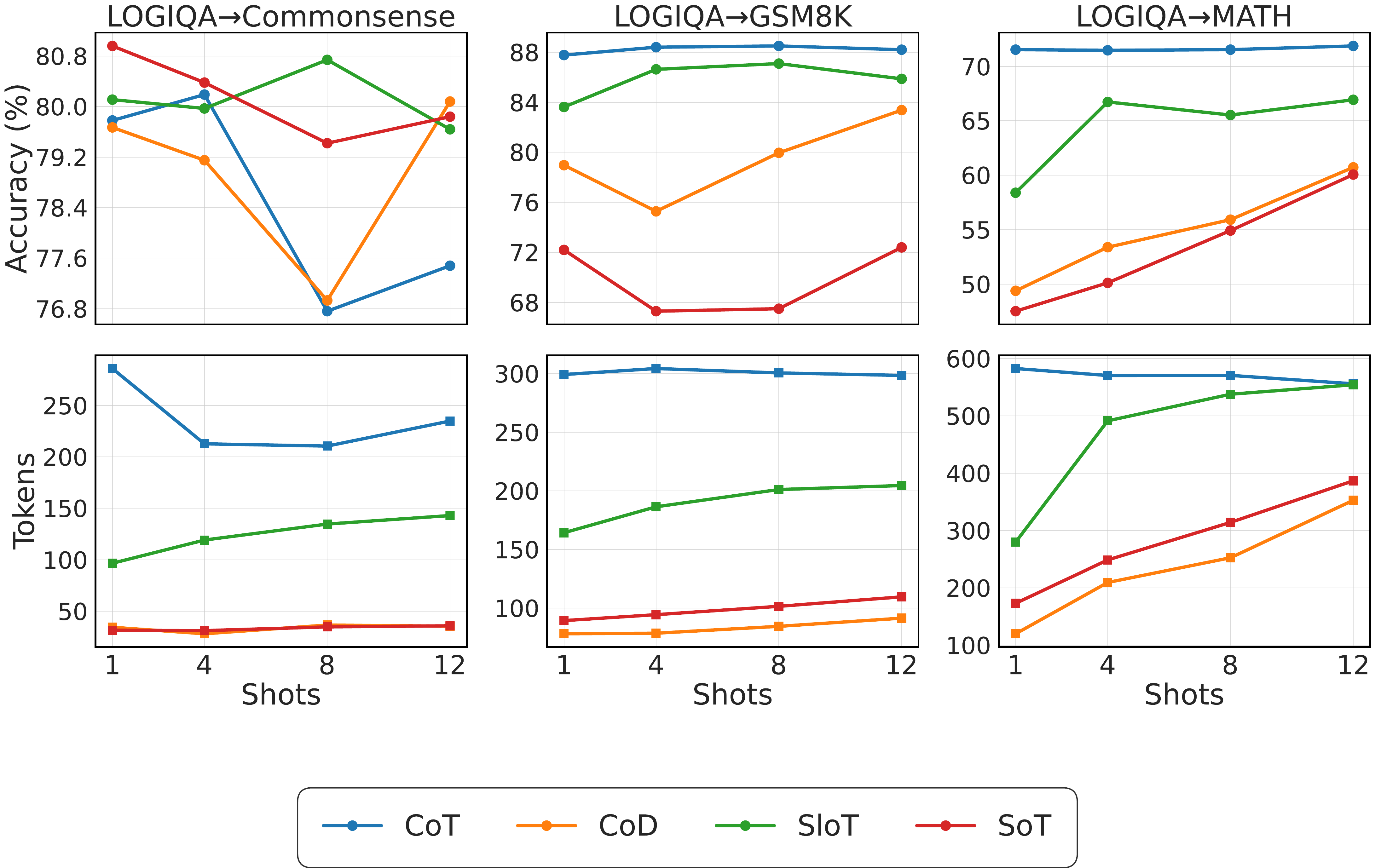}
    \caption{Comparison of few-shot transfer from \textbf{LogiQA} to Commonsense, GSM8K, and MATH500 in terms of accuracy (top) and tokens (bottom).}
    \label{fig:placeholder}
\end{figure}

\section{Conclusion}
We present EffiReason-Bench, a unified benchmark for efficient reasoning with E$^3$-Score, a CES-inspired metric that jointly measures accuracy and token efficiency. Evaluating six open-source backbones across mathematics, commonsense, and logic tasks, we find that train-free blueprint methods (CoD/SoT) achieve extreme compression but frequently sacrifice accuracy, while the train-based O1-Pruner preserves or improves accuracy with modest efficiency gains. Post-hoc TokenSkip refinement delivers favorable trade-offs on commonsense, logic, and LLaMA backbones, yet degrades MATH500 accuracy substantially on Qwen, revealing strong architecture and domain sensitivity. Among dynamic execution methods, latent-space Soft Thinking exhibits consistent robustness, whereas structure-first SloT shows instability in zero-shot settings but scales effectively with additional shots and curated exemplars. By standardizing evaluation and releasing reproducible implementations, EffiReason-Bench establishes a rigorous foundation for fair comparison and development of adaptive, efficient reasoning strategies.

\bibliography{custom}

\appendix

\section{Noise Injection in Few-shot Prompts.}
\label{noise}

\textbf{(1) Reasoning-step noise.} Applied to the reasoning chains of exemplars, including:  
(a) random deletion of reasoning steps,  
(b) random permutation of step order,  
(c) random repetition of certain steps.  

\textbf{(2) Semantic noise.}  
For mathematical reasoning tasks (GSM8K, MATH500), we randomly perturb numeric values by small offsets.  
For commonsense and logical reasoning tasks, we reverse logical connectives according to a fixed mapping:
\begin{verbatim}
reversal_map = {
   'therefore': 'however',
   'thus': 'but',
   'so': 'although',
   'because': 'despite',
   'since': 'although',
   'then': 'otherwise',
   'implies': 'contradicts',
   'leads to': 'prevents',
}
\end{verbatim}
This modification preserves fluency but alters semantic validity, making it a strong stress test.  
Noise is injected with equal probability across exemplars unless otherwise specified. By introducing these perturbations, we evaluate not only the efficiency–effectiveness trade-off but also the robustness of efficient reasoning methods under imperfect supervision. 

\section{Proof of Theorem 1}
\begin{proof}[Proof of Theorem~\ref{thm:rho-sensitivity}]
Write the weighted power mean
\begin{equation}
    \begin{split}
        M_\rho(a,b)\;&=\;\bigl(w\,a^\rho+(1-w)\,b^\rho\bigr)^{1/\rho},\\ &a,b>0,\ w\in(0,1),
    \end{split}
\end{equation}
so that $E_3(A,T;\rho)=M_\rho(a,b)$. Let $x_1=\log a$, $x_2=\log b$, $\Delta=\lvert x_1-x_2\rvert$, and define
\begin{equation}
\begin{split}
    S(\rho)&=w\,e^{\rho x_1}+(1-w)\,e^{\rho x_2},\\
    L(\rho)=&\log M_\rho(a,b)=\frac{1}{\rho}\log S(\rho)
\end{split}
\end{equation}
for $\rho\neq 0$; set $M_0(a,b)=a^w b^{\,1-w}$ (the continuous extension at $\rho=0$), whence $L(0)=w x_1+(1-w)x_2$.

Let $f(\rho)=\log S(\rho)$. A direct differentiation gives
\begin{equation}
    \begin{split}
        f'(\rho)=&\frac{w e^{\rho x_1}x_1+(1-w)e^{\rho x_2}x_2}{S(\rho)}=\mathbb{E}_{\pi_\rho}[X],\\
f''(\rho)&=\mathrm{Var}_{\pi_\rho}[X]\ge 0,
    \end{split}
\end{equation}
where $\pi_\rho$ is the softmax distribution on $\{x_1,x_2\}$ with weights proportional to $w e^{\rho x_1}$ and $(1-w)e^{\rho x_2}$. Since $X\in[\min\{x_1,x_2\},\max\{x_1,x_2\}]$ with range $\Delta$, we have the uniform variance bound
\begin{equation}
    0\ \le\ f''(\rho)=\mathrm{Var}_{\pi_\rho}[X]\ \le\ \frac{\Delta^2}{4}\qquad(\forall\,\rho\in\mathbb{R}).
\end{equation}

For $\rho\neq 0$, differentiate $L(\rho)=f(\rho)/\rho$ and set $\phi(\rho)=\rho f'(\rho)-f(\rho)$:
\begin{equation}
    \begin{split}
        L'(\rho)=&\frac{f'(\rho)}{\rho}-\frac{f(\rho)}{\rho^2}=\frac{\phi(\rho)}{\rho^2},\\
&\phi'(\rho)=\rho f''(\rho),\\ \phi(0)=0\ ( &\text{since }f(0)=\log(w+1-w)=0 ).
    \end{split}
\end{equation}

Hence $\phi(\rho)=\int_0^\rho t\,f''(t)\,dt$ and therefore
\begin{equation}
    L'(\rho)=\frac{1}{\rho^2}\int_0^\rho t\,f''(t)\,dt.
\end{equation}
Using $f''(t)\le \Delta^2/4$ and changing variables if $\rho<0$,
\begin{equation}
    \bigl|L'(\rho)\bigr|\ \le\ \frac{1}{\rho^2}\int_0^{|\rho|} t\,\frac{\Delta^2}{4}\,dt
=\frac{\Delta^2}{8}\qquad(\rho\neq 0).
\end{equation}
At $\rho=0$, taking the limit in the integral representation yields
\begin{equation}
        L'(0)=\tfrac12 f''(0)=\tfrac12\,\mathrm{Var}_{\pi_0}[X]\ \le\ \frac{\Delta^2}{8}.
\end{equation}
Thus $\lvert L'(\rho)\rvert\le \Delta^2/8$ for all $\rho\in\mathbb{R}$. By the mean‑value theorem,
\begin{equation}
    \bigl|L(\rho_1)-L(\rho_2)\bigr|\ \le\ \frac{\Delta^2}{8}\,|\rho_1-\rho_2|,
\end{equation}
which is the claimed Lipschitz bound on $\log E_3$. Exponentiating gives the multiplicative bounds. 

\emph{Tightness.} With $w=\tfrac12$ and $x_1=-x_2=\Delta/2$, we have $L'(0)=\Delta^2/8$, so the constant cannot be improved in general.
\end{proof}

\end{document}